\documentclass[sigconf]{acmart}

\usepackage{balance}

\def\BibTeX{{\rm B\kern-.05em{\sc i\kern-.025em b}\kern-.08emT\kern-.1667em\lower.7ex\hbox{E}\kern-.125emX}}
    
\copyrightyear{2019} 
\acmYear{2019} 
\setcopyright{acmlicensed}
\acmConference[KDD '19]{The 25th ACM SIGKDD Conference on Knowledge Discovery and Data Mining}{August 4--8, 2019}{Anchorage, AK, USA}
\acmBooktitle{The 25th ACM SIGKDD Conference on Knowledge Discovery and Data Mining (KDD '19), August 4--8, 2019, Anchorage, AK, USA}
\acmPrice{15.00}
\acmDOI{10.1145/3292500.3330898}
\acmISBN{978-1-4503-6201-6/19/08}

\usepackage[labelformat=simple]{subcaption}

\usepackage{booktabs} 
\usepackage{amsfonts}       
\usepackage{nicefrac}       
\usepackage{microtype}      
\usepackage{epsf, latexsym}
\usepackage{xspace, amsmath, amssymb, bm}
\usepackage{color,amsthm}
\usepackage{graphicx}
\usepackage{color}
\usepackage{tabularx}
\usepackage{varwidth}

\usepackage{bbm}
\usepackage{algorithm}
\usepackage[noend]{algpseudocode}
\newtheorem{thm}{Theorem}

\newtheorem{prop}{Proposition}

\newtheorem{defn}{Definition}

\def\QED{\mbox{\rule[0pt]{1.5ex}{1.5ex}}}
\def\endproof{\hspace*{\fill} \QED\par\endtrivlist\unskip}

\newcolumntype{R}{>{\raggedleft\arraybackslash}X}
\newcommand{\pluseq}{\mathrel{+}=}

\def\x{{{\mathbf x}}}

\def\F{\ensuremath{{\mathcal F}}\xspace}

\def\P{\ensuremath{{\mathbb P}}\xspace}
\def\D{\ensuremath{{\mathcal D}}\xspace}
\def\E{\ensuremath{{\mathbb E}}\xspace}

\def\R{\ensuremath{{\mathbb R}}\xspace}

\def\X{\ensuremath{{\mathcal X}}\xspace}
\def\Y{\ensuremath{{\mathcal Y}}\xspace}

\DeclareMathOperator*{\argmin}{argmin}
\DeclareMathOperator*{\argmax}{argmax}

\begin{document}

\title{Axiomatic Interpretability for Multiclass Additive Models}
\author{Xuezhou Zhang}
\affiliation{University of Wisconsin-Madison}
\email{xzhang784@wisc.edu}

\author{Sarah Tan}
\affiliation{Cornell University}
\email{ht395@cornell.edu}

\author{Paul Koch}
\affiliation{Microsoft Research}
\email{paulkoch@microsoft.com}

\author{Yin Lou}
\affiliation{Ant Financial}
\email{yin.lou@antfin.com}

\author{Urszula Chajewska}
\affiliation{Microsoft}
\email{urszc@microsoft.com}

\author{Rich Caruana}
\affiliation{Microsoft Research}
\email{rcaruana@microsoft.com}
\email{} 
\email{} 

%
%

\begin{abstract}
Generalized additive models (GAMs) are favored in many regression and binary classification problems because they are able to fit complex, nonlinear functions while still remaining interpretable. In the first part of this paper, we generalize a state-of-the-art GAM learning algorithm based on boosted trees to the multiclass setting, showing that this multiclass algorithm outperforms existing GAM learning algorithms and sometimes matches the performance of full complexity models such as gradient boosted trees. 

In the second part, we turn our attention to the interpretability of GAMs in the multiclass setting. Surprisingly, the natural interpretability of GAMs breaks down when there are more than two classes. Naive interpretation of multiclass GAMs can lead to false conclusions. Inspired by binary GAMs, we identify two axioms that any additive model must satisfy in order to not be visually misleading. We then develop a technique called Additive Post-Processing for Interpretability (API) that provably transforms a pretrained additive model to satisfy the interpretability axioms without sacrificing accuracy. The technique works not just on models trained with our learning algorithm, but on any multiclass additive model, including multiclass linear and logistic regression. We demonstrate the effectiveness of API on a 12-class infant mortality dataset.
\end{abstract}

\maketitle
\section{Introduction}
\label{sec:introduction}
Interpretable models, though sometimes less accurate than black-box models, are preferred in many real-world applications. In criminal justice,
finance, hiring, and other domains that impact people's lives, interpretable models are often used because their transparency helps determine if a model is biased or unsafe \cite{zeng2016interpretable,tan2017detecting}. And in critical applications such as healthcare, where human experts and machine learning models often work together, being able to understand, learn from, edit and trust the learned model is also important \cite{caruana2015intelligible, Holstein2018practitioner}.
\begin{figure}[t!]
	\centering
	\begin{subfigure}[t]{0.49\columnwidth}
		\centering
		\includegraphics[width=1.0\columnwidth]{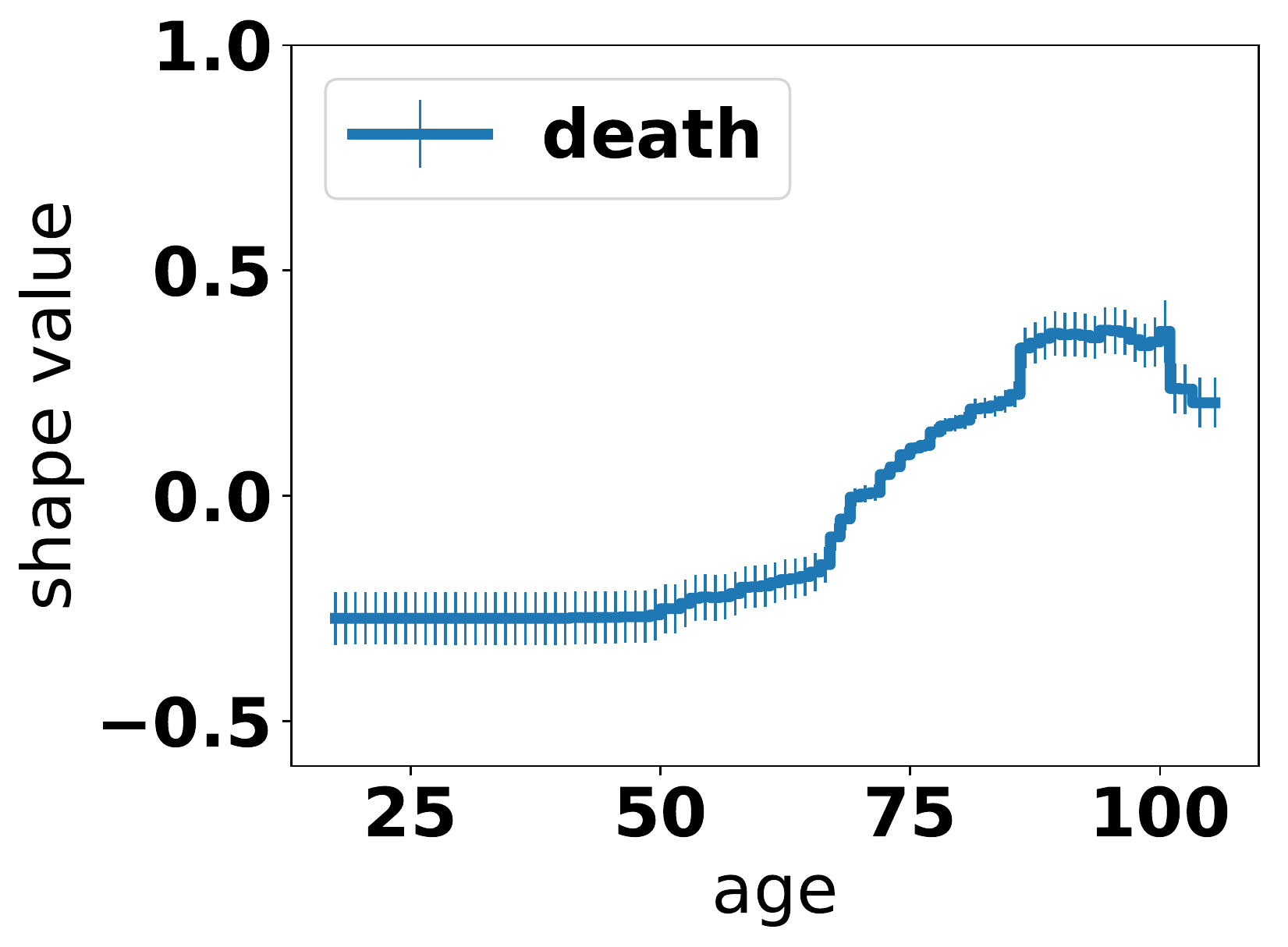} 
		\caption{Binary GAM age shape}
		\label{fig:age_binary}
	\end{subfigure}
	~
	\begin{subfigure}[t]{0.49\columnwidth}
		\centering
		\includegraphics[width=1.0\columnwidth]{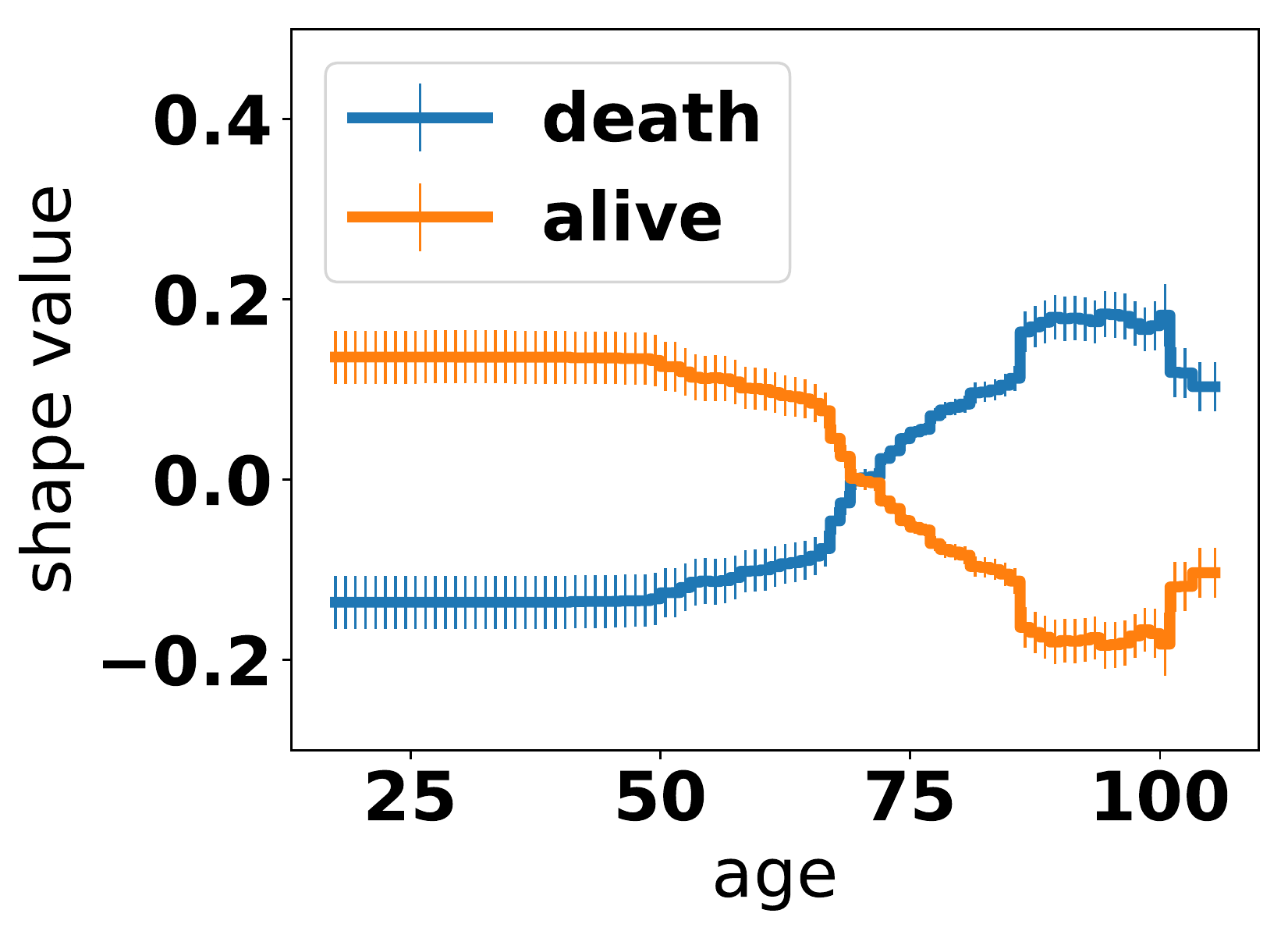}
		\caption{Multiclass GAM age shape}
		\label{fig:age_mtc}
	\end{subfigure}
	\caption{\small Shape functions for age in the pneumonia problem \cite{caruana2015intelligible}.}
	\label{fig:age}
	\vspace*{-0.1in}
\end{figure}
Generalized additive models (GAMs) are among the most powerful interpretable models when individual features play major effects \cite{hastie1990generalized,lou2012intelligible}. In the binary classification setting, we consider standard GAMs with logistic probabilities:
$\hat\P(Y=1)=(1+\exp(-F(x)))^{-1}$, where the \textit{logit} $F(x)$ is an additive function of individual features:
\begin{eqnarray}
F(x) = \sum_{i=1}^d f_i(x_i).
\end{eqnarray}
in which $d$ is the number of features. Here, $x_i$ is the $i$-th feature of data point $x$, and we denote $f_i$ the \textit{shape function} of feature $i$ for the positive class.
Previously, \citeauthor{lou2012intelligible} evaluated various GAM fitting algorithms, and found that gradient boosting of shallow bagged trees that cycle one-at-a-time through the features outperformed other methods on a number of regression and binary classification datasets \cite{lou2012intelligible}. Their model is called the Explainable Boosting Machine (EBM).\footnote{New code for training EBM additive models has recently been released and can be found at \url {https://github.com/microsoft/interpret}.} The first part of this paper generalizes EBMs to the multiclass setting. We consider standard GAMs with softmax probabilities:
\begin{eqnarray}\label{eq:mtc_gam}
\hat\P(Y=k) = \frac{\exp\left(F_k(x)\right)}{\sum_{j=1}^{K}\exp\left(F_j(x)\right)},
\end{eqnarray}
where the \textit{logit of class} $k$, $F_k(x)$, is also an additive function of individual features, $F_k(x) = \sum_{i=1}^df_{ik}(x_i)$ and $f_{ik}$ is the shape function of feature $i$ for class $k$. We present our multiclass GAM fitting algorithm, MC-EBM, in Section~\ref{sec:cyclic_gradient_boosting} and in Section~\ref{sec:accuracyresults} we empirically evaluate its performance on five large-scale, real-world datasets.

\begin{figure}[h!]
	\centering
	\begin{minipage}[t]{0.464\columnwidth}
		\centering
		\includegraphics[width=\columnwidth]{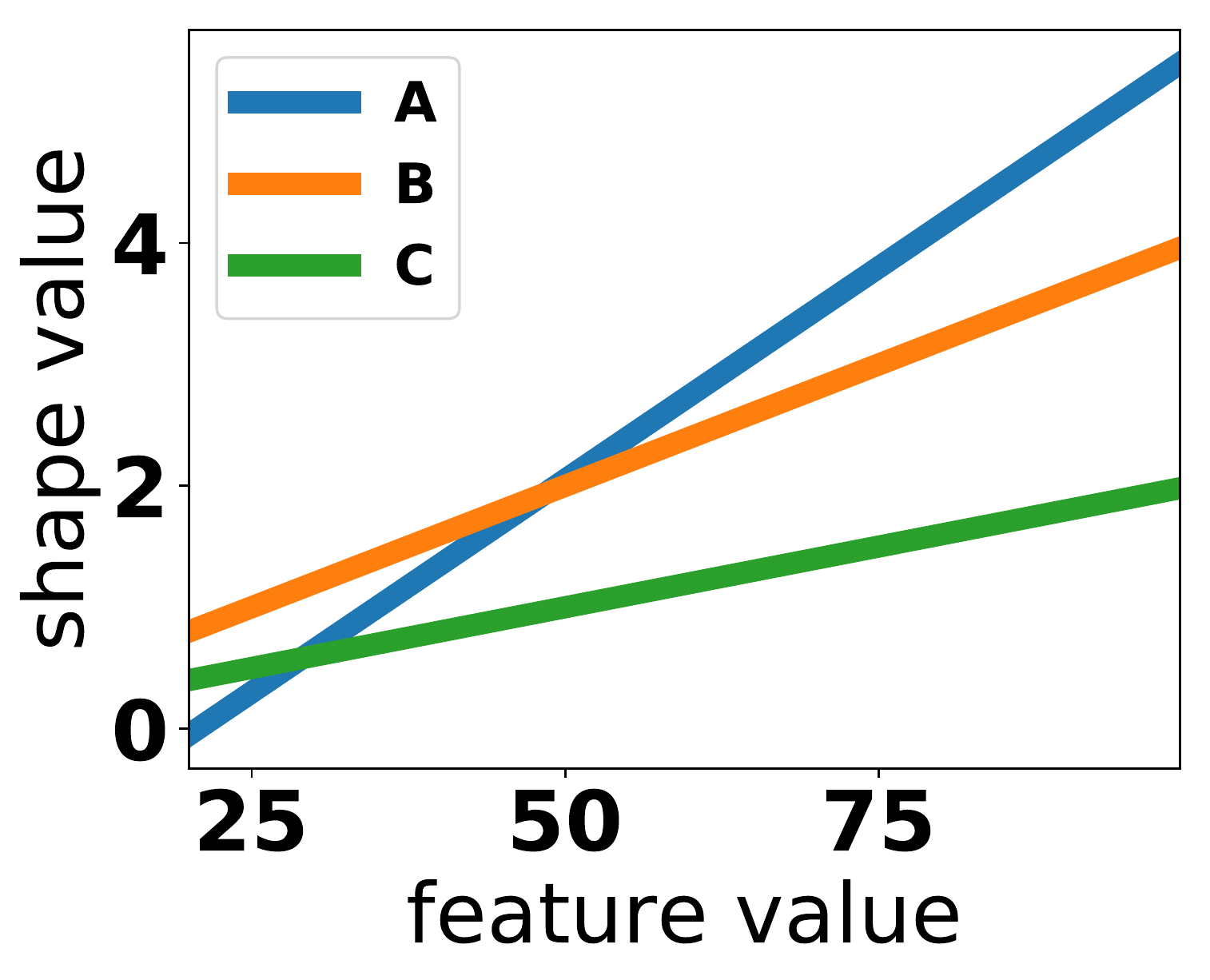} 
		\subcaption{Toy model 1}
		\label{fig:toy1}
	\end{minipage}~
	\begin{minipage}[t]{0.49\columnwidth}
		\centering
		\includegraphics[width=\columnwidth]{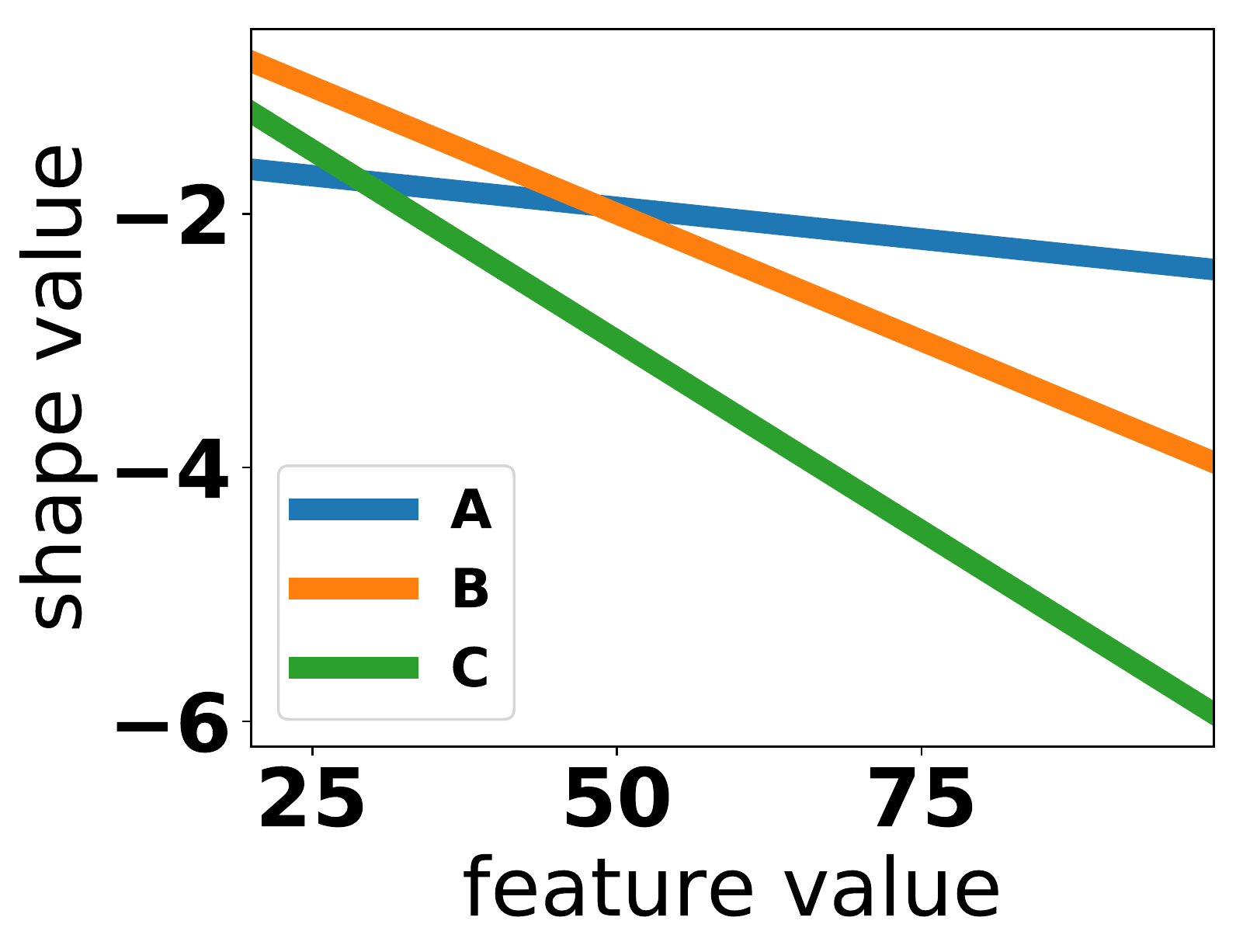} 
		\subcaption{Toy model 2}
		\label{fig:toy2}
	\end{minipage}\\
	\begin{minipage}[t]{0.49\columnwidth}
		\centering
		\includegraphics[width=\columnwidth]{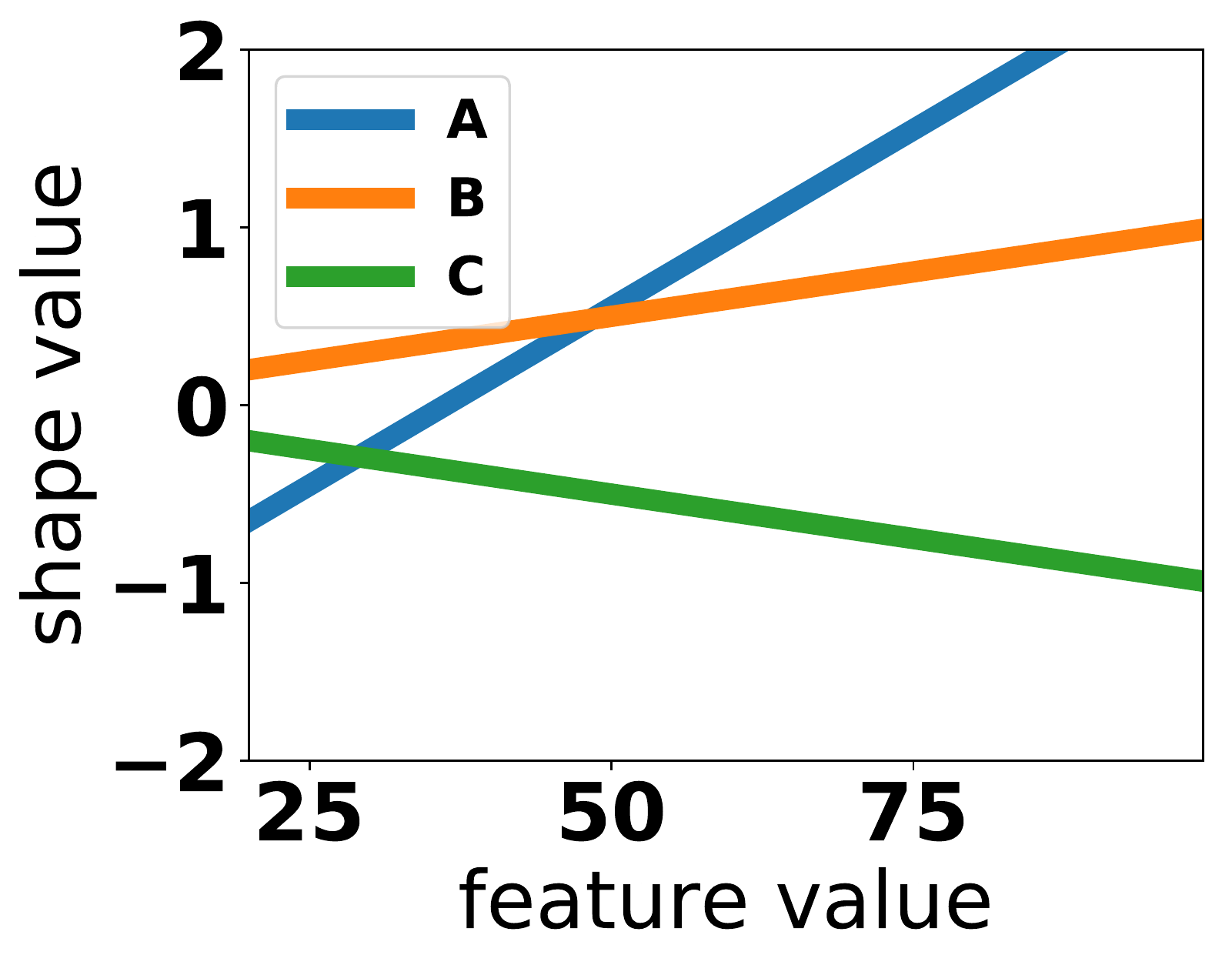} 
		\subcaption{Toy model 3}
		\label{fig:toy1}
	\end{minipage}~
	\begin{minipage}[t]{0.464\columnwidth}
		\centering
		\includegraphics[width=\columnwidth]{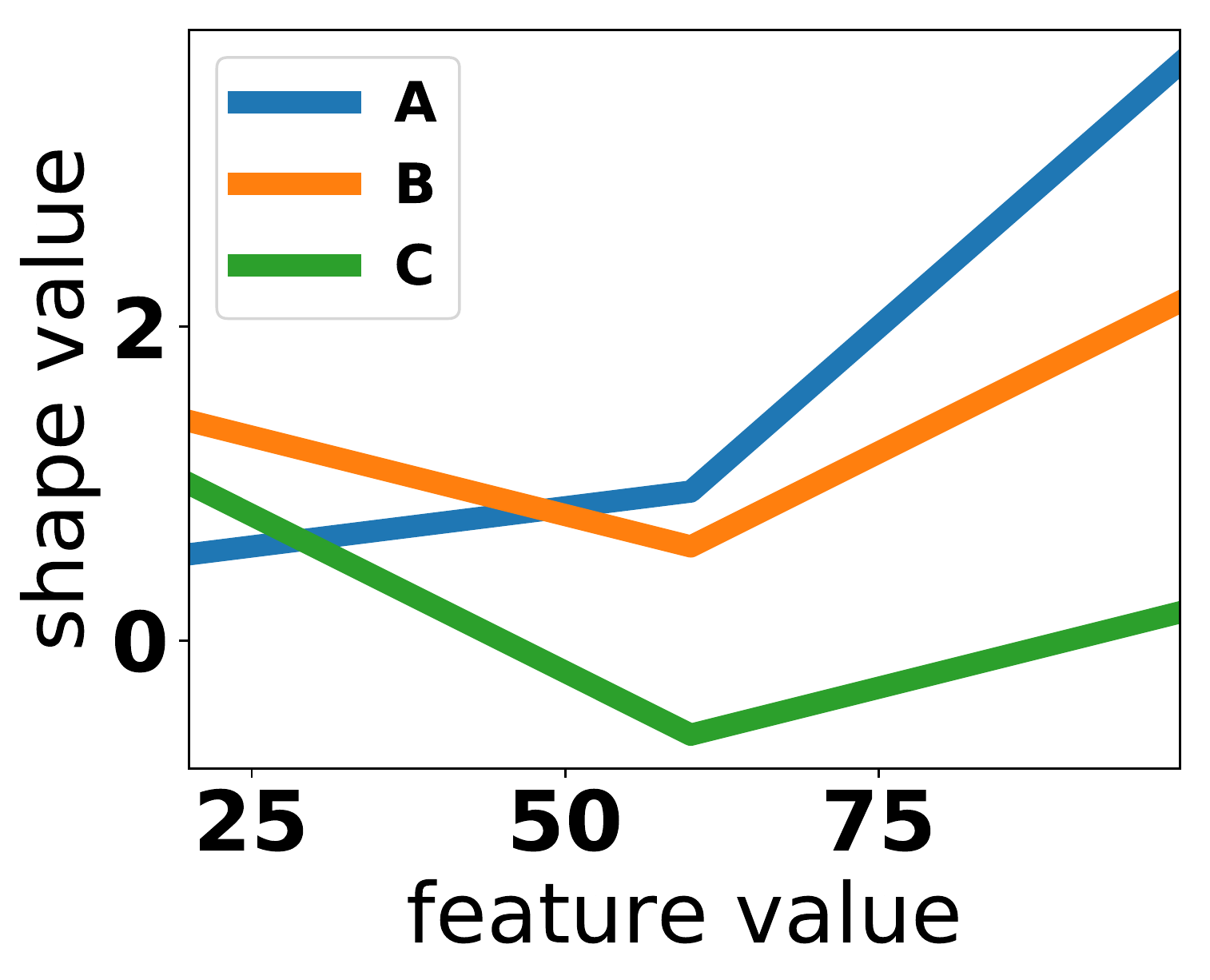} 
		\subcaption{Toy model 4}
		\label{fig:toy1}
	\end{minipage}\\
	\begin{minipage}[t]{0.49\columnwidth}
		\centering
		\includegraphics[width=\columnwidth]{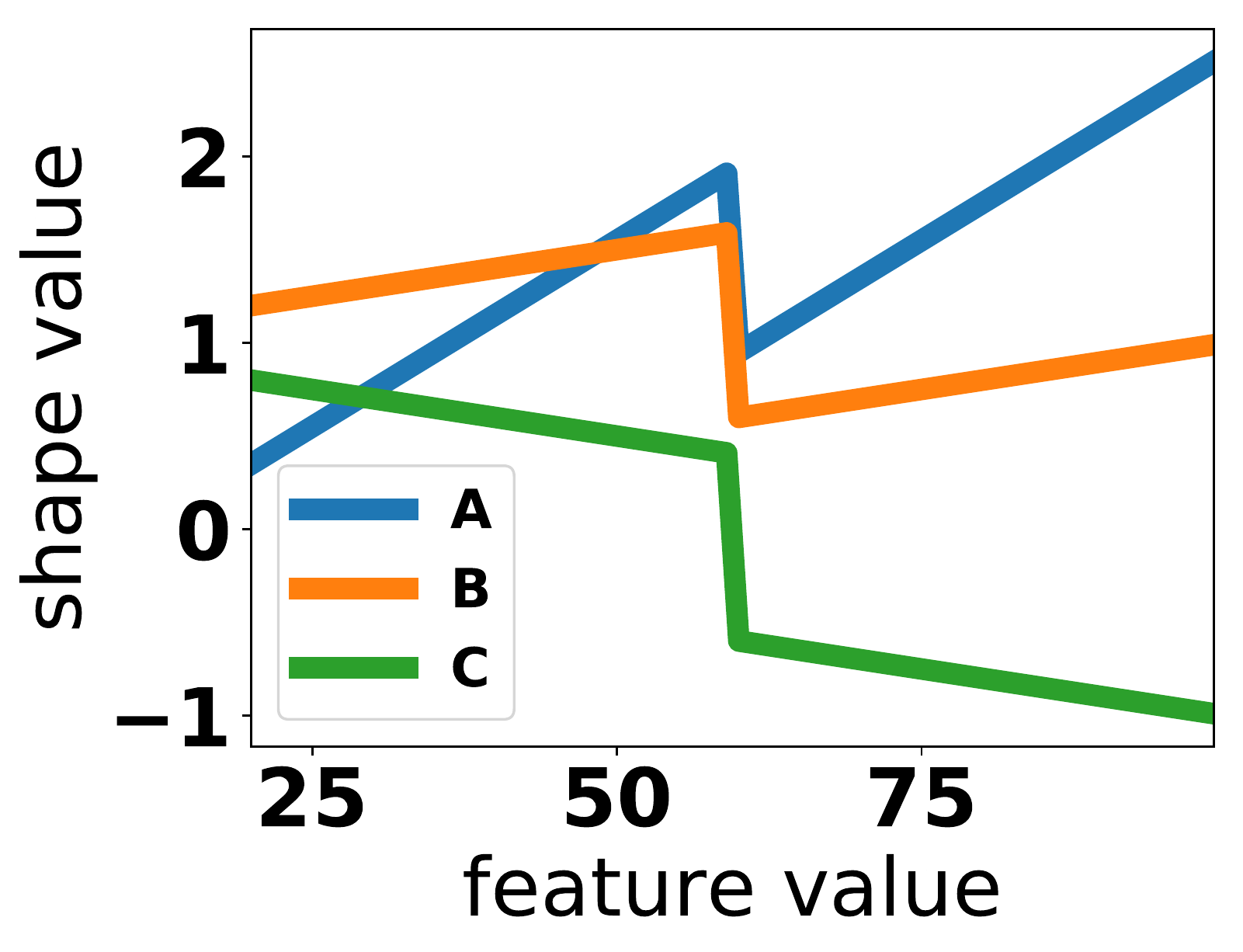} 
		\subcaption{Toy model 5}
		\label{fig:toy3}
	\end{minipage}~
	\begin{minipage}[t]{0.49\columnwidth}
		\centering
		\includegraphics[width=\columnwidth]{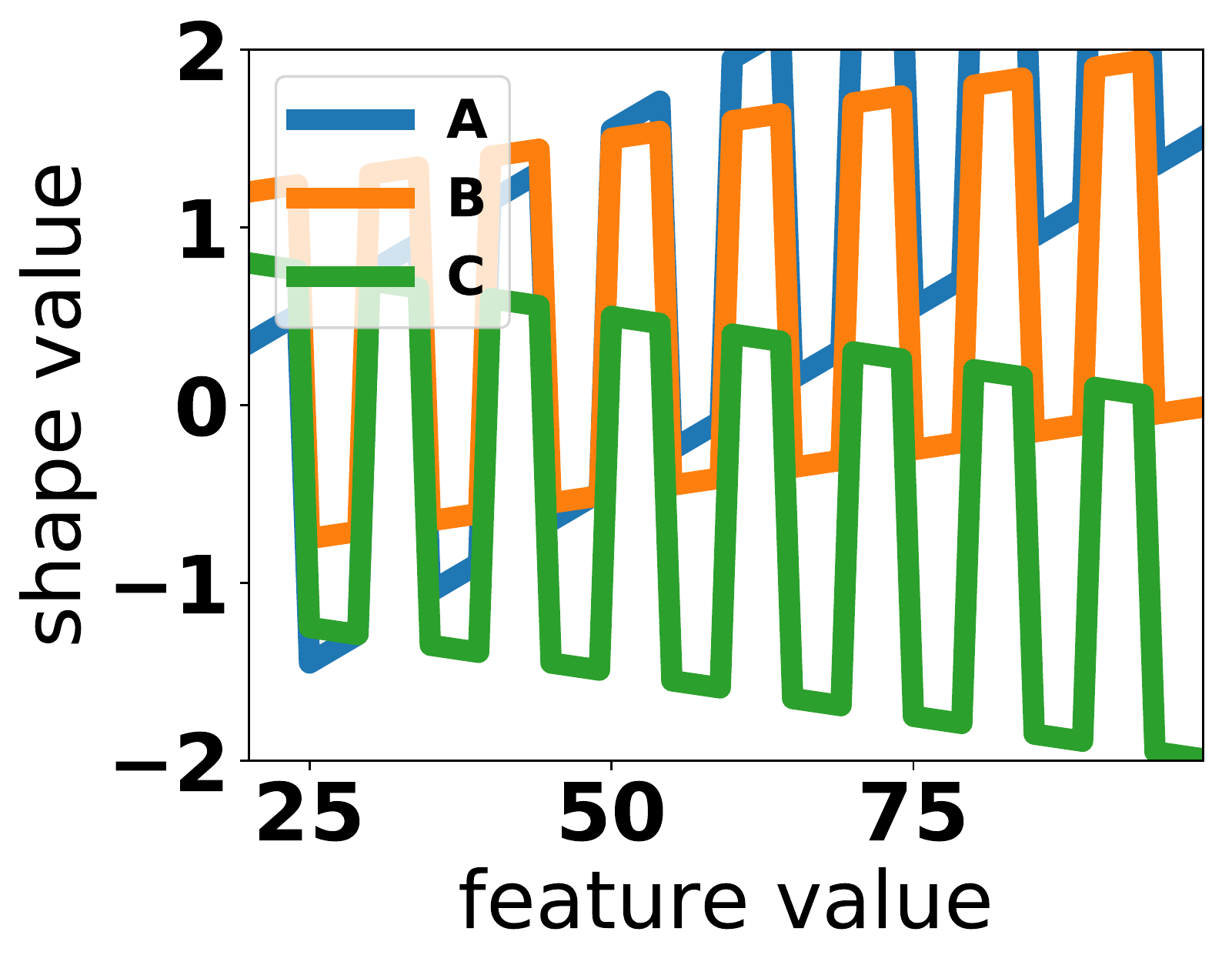} 
		\subcaption{Toy model 6}
		\label{fig:toy4}
	\end{minipage}\\
	\begin{minipage}[t]{0.49\columnwidth}
		\centering
		\includegraphics[width=\columnwidth]{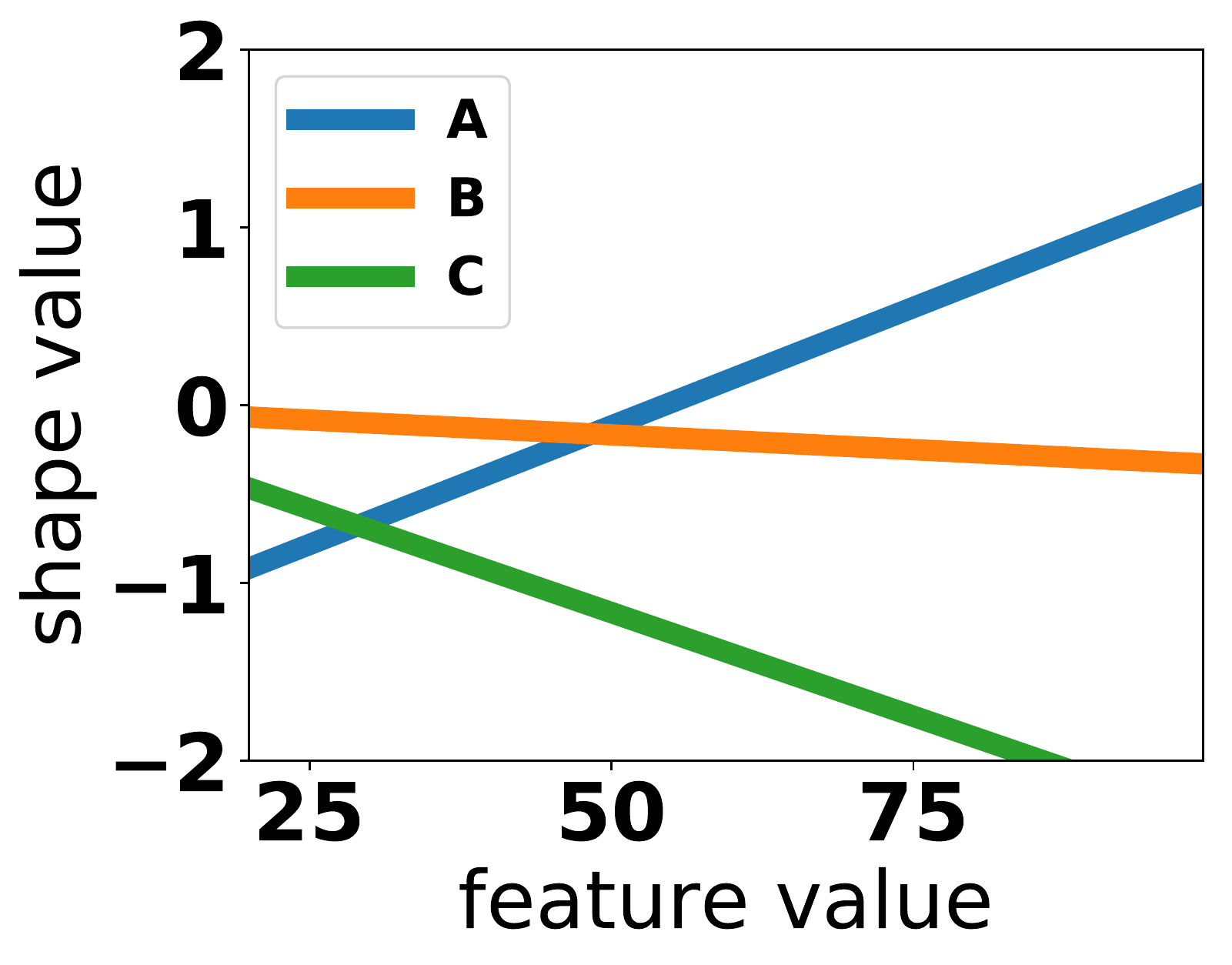} 
		\subcaption{Toy models after API}
		\label{fig:toy7}
	\end{minipage}~
	\begin{minipage}[t]{0.49\columnwidth}
		\centering
		\includegraphics[width=\columnwidth]{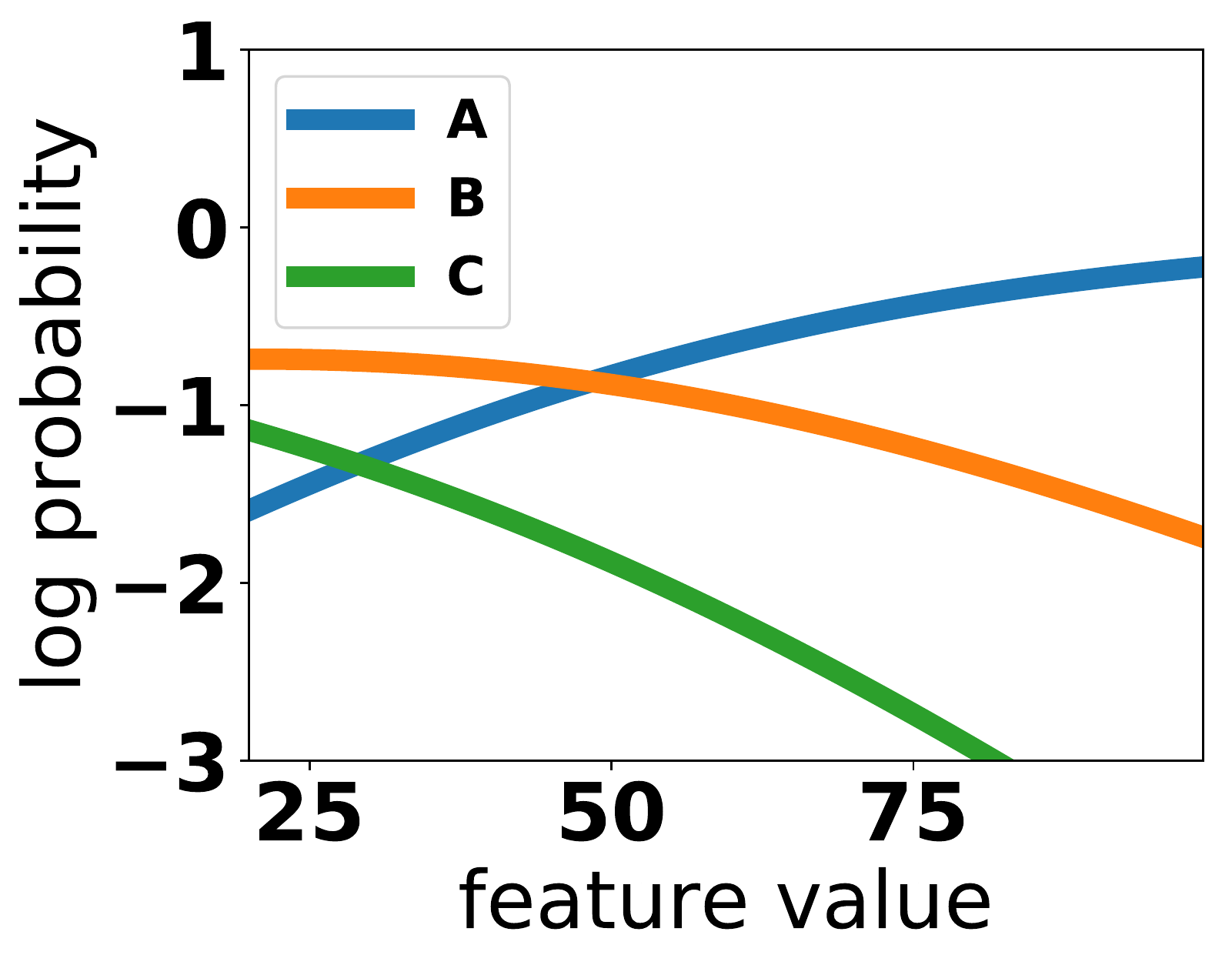} 
		\subcaption{True class probabilities}
		\label{fig:toy8}
	\end{minipage}
	\caption{GAM shape functions for a toy 3-class problem.}

	\label{fig:toy}
\end{figure}

Binary GAMs are readily interpretable because the influence of each feature $i$ on the outcome is captured by a {\em single} 1-d shape function $f_i$ that can be easily visualized.
For example, Figure~\ref{fig:age_binary} shows the relationship between age and the risk of dying from pneumonia. When interpreting shape functions like this, practitioners often focus on two key factors: the local monotonicity of the curve and the existence of discontinuities (if the feature value is continuous). For example, the `age' plot in Figure \ref{fig:age_binary} could be described by a physician as:
\begin{quote}
	``Risk is low and constant from age 18-50, rises slowly from age 50-67, then rises quickly from age 67-90. There is a small jump in risk at age 67, soon after typical retirement age, a surprising jump in risk at age 85, and a surprising drop in risk at about age 100.''
\end{quote}
In a binary logistic function, the rising, falling and ``jumps'' in each shape function faithfully correspond to the increasing, decreasing and sudden changes in the predicted probability, so this kind of summary is a faithful representation of the model's predictions.

In the multiclass setting, however, the influence of feature $i$ on class $k$ is no longer captured by a single shape function $f_{ik}$, but through the interplay of all $f_{ij}$'s, $j=1,...,K$. 
In particular, even if the logit for class $k$ is increasing, the probability for class $k$ might still decrease if the logits for other classes increase more rapidly. As a result, the learned shape functions, if presented without post-processing, can be visually misleading.
For example, Figures \ref{fig:toy}a-f show the shape functions of six toy GAM models with three classes and only one feature. Each model appears to have very different shape functions: \ref{fig:toy}(a) all rising, \ref{fig:toy}(b) all falling, \ref{fig:toy}(c) some falling, some rising,  \ref{fig:toy}(d) 2-of-3 falling, then all 3 rising,  \ref{fig:toy}(e) big drop in the middle,  \ref{fig:toy}(f) oscillating. {\em Interestingly, however, all six models make identical predictions.} Because these models have only one feature, we can actually plot the predicted probabilities as functions of the feature value (this is not possible with more than one feature). In Figure \ref{fig:toy}(h), one can see that class $A$'s probability is monotonically increasing, while class $B$ and $C$'s probabilities are monotonically decreasing, which is vastly different from any of the shape functions (a)-(f). If a domain expert examines the shape functions in \ref{fig:toy}(c), she/he is likely to be misled to believe that the predicted probabilities for both class A and B are increasing and only the predicted probability for class C is decreasing, which is inconsistent with ground truth.
This representation problem, if not solved, greatly reduces the interpretability of GAMs in multiclass problems.

The second half of this paper focuses on mitigating the misleadingness of multiclass GAM shapes. We start by examining how users interpret binary GAMs and identify a set of interpretability axioms --- criteria that GAM shapes should satisfy to guarantee interpretability. We then present several properties of additive models that make it possible to regain interpretability.
Making use of these properties, we design a method, Additive Post-Processing for Interpretability (API), that provably transforms any pretrained additive model to satisfy the axioms of interpretability \textbf{without} sacrificing any predictive accuracy. Figure \ref{fig:toy}(g) shows the shape functions that result from passing any of the models \ref{fig:toy}(a)-(f) through API. After API post-processing, the new canonical shape functions successfully match the probability trends for the corresponding classes in Figure \ref{fig:toy}(h) and are no longer misleading.

\section{Related Work}
Generalized additive models (GAMs) were first introduced (in statistics) 
to allow individual features to be modeled flexibly \cite{hastie1990generalized,wood2006generalized}. They are traditionally fitted using splines \cite{eilers1996}. 
Other base learners include trees \cite{lou2012intelligible}, 
trend filters \cite{tibshirani2014adaptive}, wavelets \cite{wand2011penalized}, etc. 

Comparing several different GAM fitting procedures including backfitting and simultaneous optimization,
\citeauthor{binder2008comparison} found that boosting performed particularly well in high-dimensional settings \cite{binder2008comparison}. \citeauthor{lou2012intelligible} developed the Explainable Boosting Machine (EBM) \cite{lou2012intelligible,lou2013accurate} which boosts shallow bagged tree base learners by repeatedly cycling through the available features. This paper generalizes EBM to the multiclass setting. 

We briefly review other available GAM software: \texttt{mboost} \cite{mboost} 
fits GAMs using component-wise gradient boosting \cite{buhlmann2003boosting}; 
\texttt{pyGAM} \cite{pyGAM}
fits GAMs with P-splines base learners using penalized iteratively reweighted least
squares. However, neither supports multiclass classification. \texttt{mgcv} \cite{mgcv}, a widely-used R package, fits GAMs with spline-based learners using penalized likelihood  and supports multiclass classification but is not scalable (cf. Section \ref{sec:accuracyresults} for more details).
To the best of our knowledge, our package is the first that can train large-scale, high-performance multiclass GAMs. 

Our work is also closely related to recent developments in interpretable machine learning. We distinguish between several lines of research that aim to improve the interpretability of machine learning models. The first line of work aims to explain the predictions of a black-box model, either locally \cite{lime,baehrens2010explain} or globally \cite{ribeiro2018anchors,tan2018transparent}. Another line of research aims at building interpretable models from the ground-up, such as rule lists \cite{letham2015,yang2017scalable}, scoring systems \cite{zeng2016interpretable}, decision sets \cite{lakkaraju2016interpretable}, and additive models \cite{lou2013accurate}. Finally, a third line of research tries to improve the interpretability of black-box models by regularizing their internal representations or explanations \cite{ross2017right, alvarez2018towards}. The majority of these works, however, focus on binary classification and regression. This paper is one of the first to address interpretability challenges in the multiclass setting.

It is worth pointing out that these various lines of work are fundamentally different and based upon different beliefs \cite{doshi2017towards,lipton2016mythos}. The first line of work is built upon the belief that it is sometimes acceptable to use black-box models that are not themselves interpretable, but where human users can understand how the black-box predictions/decisions were made with the help of explanation tools. The second line of work is built upon the belief that there is value in fully interpretable/transparent models even though black-box models might sometimes yield higher accuracy. As a result, although these lines of work are all concerned with interpretability, they cannot be easily compared. 

Because of the lack of other multiclass interpretable models to compare against, and because of the difficulty of comparing interpretable models with explanation methods, this paper focusses solely on interpretability within the GAM model class. 

\section{Notation and Problem Definition}
\label{sec:notation}
In this section, we define notation that will be used throughout the paper.
We focus on multiclass classification in which $\X\in \R^d$ is the input space and $\Y = [K]$ is the output space, where $K$ is the number of classes and $[K]$ denotes the set $\{1,...,K\}$. Let $\D = \{\x_n,y_n\}_{n=1}^N$ denote a training set of size $N$, where $\x_n = (x_{n1},...,x_{nd})\in\X$ is a feature vector with $d$ features and $y_n\in\Y$ is the target. For $k\in[K]$, let $p_k$ denote the empirical proportion of class $k$ in $\D$. Given a model $\Theta$, let $\Theta(\x_n)$ denote the prediction of the model on data point $\x_n$. Our learning objective is to minimize the expected value of some loss function $L(y,\Theta(\x))$. In multiclass classification, the model output is a probability distribution among the $K$ classes, $\hat\P(Y=k), k\in[K]$. We will be using the multiclass cross entropy loss defined as: 
\begin{eqnarray}\label{eq:loss}
L(y,\Theta(\x)) =-\sum_{k\in[K]} \mathbbm{1}_{y=k}\log\hat\P(Y=k).
\end{eqnarray}
We focus on GAM models of the form \eqref{eq:mtc_gam} with softmax probabilities. We denote $\F = \{f_{ij}: i\in[d], j\in[K]\}$ as the set of shape functions for a multiclass GAM model, and also as the model itself. Throughout the paper, we make the following assumptions of the multiclass shape functions $f_{ij}$'s. For continuous feature $i$, $f_{ij}$'s domain is a continuous finite interval $[a,b]$; for categorical or ordinal features, $f_{ij}$'s domain is a finite ordered discrete set.
Notice that we are enforcing an ordering on the otherwise unordered categorical variables in order to visualize the shape functions in a deterministic order. We denote the domain of feature $i$ as $X_i$. For the API post-processing method (Section \ref{sec:api}), we also assume that the shape functions $f_{ij}$ of continuous features are continuous everywhere except for a finite number of points. Note that this is a weak assumption, as most base learners used for fitting GAM shapes satisfy this assumption (e.g., splines are continuous and trees are piece-wise constant with a finite number of discontinuities). Finally, we overload the $\nabla$ operator as follows: In the continuous domain, $\nabla_x f = \lim_{\Delta x\rightarrow 0}\frac{f(x+\Delta x)-f(x)}{\Delta x}$ when $f_{ij}$'s are all continuous at $x$; $\nabla_x f =f(x^+)-f(x^-)$ when some $f_{ij}$'s are discontinuous at $x$. In the discrete domain, $\nabla_x f =f(x_{next})-f(x)$, where $x_{next}$ denotes the immediate next value.

\section{Multiclass GAM Learning via Cyclic Gradient Boosting}
We now describe the training procedure for MC-EBM, our generalization of binary EBM \cite{lou2012intelligible} to the multiclass setting. We use bagged trees as the base learner for boosting, with largest variance reduction as the splitting criterion. We control tree complexity by limiting the number of leaves $L$.
\subsection{Cyclic Gradient Boosting}
\label{sec:cyclic_gradient_boosting}
Our optimization procedure is cyclic gradient boosting \cite{buhlmann2003boosting,lou2012intelligible}, a variant of standard gradient boosting \cite{friedman2001greedy} where features are cycled through sequentially to learn each individual shape function. The algorithm is presented in Algorithm \ref{alg}.

In standard gradient boosting, each boosting step fits a base learner to the pseudo-residual, the negative gradient in the functional space \cite{friedman2001greedy}. 
In a multiclass setting with cross entropy loss \eqref{eq:loss} and softmax probabilities \eqref{eq:mtc_gam}, the pseudo-residual for class $j$ is:
\begin{eqnarray}
\tilde y_{j} = -\frac{\partial L(y,\{\hat\P(Y=j)\}_{j=1}^K)}{\partial F_{j}} = \mathbbm{1}_{y=j} - \hat\P(Y=j).\nonumber
\end{eqnarray}

Adding the fitted base learner (multiplied by a typically small constant $\eta$) to the ensemble corresponds to taking an approximate gradient step in the functional space with learning rate $\eta$.  
However, as suggested by \citeauthor{friedman2000additive} \cite{friedman2000additive}, to speed up computation one can instead take an approximate Newton step using a diagonal approximation to the Hessian. 
The resulting additive update to learn a multiclass GAM then becomes:
\begin{eqnarray}
f_{ik}^+ &=&  f_{ik} + \eta\sum_{l\in[L]} \gamma_{ilk}\mathbbm{1}_{x_i\in R_{il}}, \text{ where}\\
\gamma_{ilk} &=& \frac{K-1}{K}\frac{\sum_{\x\in R_{il}}\tilde y_{ik}}{\sum_{\x\in R_{il}}|\tilde y_{ik}|(1-|\tilde y_{ik}|)},\label{eq:gamma}
\end{eqnarray}
for $i\in[d], k\in[K], l\in[L]$, where $R_{il}$ is the set of training points in tree leaf $l$ for current feature $i$. Applying the above boosting procedure cyclically to individual features gives our multiclass cyclic boosting algorithm (Algorithm \ref{alg}).

\begin{algorithm}[h!]
	\caption{Multiclass GAM Learning via Cyclic Gradient Boosting (MC-EBM)}\label{alg}
	\begin{algorithmic}[1]
		\State $f_{ij} \gets 0$, for $i\in[d]$, $j\in[K]$
		\For{$m = 1$ to $M$}
		\For{$i = 1$ to $d$}
		\State \begin{varwidth}[t]{200pt}$\tilde y_{nj} \gets \mathbbm{1}_{y_n=j} - \hat\P(Y=j\vert X = \x_n)$, $n\in[N]$, $j\in[K]$.\end{varwidth}
		\For{$b = 1$ to $B$}
		\State \begin{varwidth}[t]{175pt} Create bootstrap sample $b$ from the training set $\left\{(x_n,\tilde y_{n})\right\}_{n=1}^N$.\end{varwidth}
		\State \begin{varwidth}[t]{175pt}Learn tree $\{R_{ilb}\}$ with $L$ leaf nodes on bootstrap sample $b$.\end{varwidth}
		\State Compute $\gamma_{iljb}$ using equation \eqref{eq:gamma}.
		\EndFor
		\State  \begin{varwidth}[t]{175pt}$f_{ij}\pluseq \eta\sum_{l=1}^L\left[\frac{1}{B}\sum_{b=1}^B \gamma_{iljb}\mathbbm{1}_{x_i\in R_{ilb}}\right]$, for $j = 1,...,K$.\end{varwidth}
		\EndFor
		\EndFor
	\end{algorithmic}
\end{algorithm}

\subsubsection{Hyperparameters.}
We found the following hyperparameters for MC-EBM to be high performing across all datasets: learning rate $\eta = 0.01$, number of leaves in tree $L=3$, number of bagged trees in each base learner $B=100$, number of boosting iterations $M=5,000$  with early stopping based on held-out validation loss. These are the default hyperparameter choices in \texttt{InterpretML}. 

\subsection{Accuracy on Real Datasets}
\label{sec:accuracyresults}
In this section, we evaluate MC-EBM against other multiclass baselines. We select five datasets with interpretable features and different numbers of classes, features, and data points. 
Table \ref{tab:dataset} describes them. Diabetes, Covertype, Sensorless and Shuttle are from the UCI repository; Infant Mortality (IM) is from the Centers for Disease Control and Prevention \cite{cdcIM}. We use normalized Shannon entropy $H= -(\sum_{k\in[K]} p_k\log p_k)/K$ to report the degree of imbalance in each dataset: $H=1$ indicates a perfectly balanced dataset (same number of points per class) while $H=0$ denotes a perfectly unbalanced dataset (all points in one class). For the IM dataset, due to its extreme class imbalance (more than 99\% of the data belongs to the `alive' class), we perform a 1\% downsampling of the `alive' data for accuracy comparison. Later, in Section \ref{sec:interpretability}, we use the whole IM dataset to train an MC-EBM model as a case study for multi-class interpretability.
\begin{table}[h!]
     \begin{tabularx}{\columnwidth}{ r r r R R }
     \toprule
     \textbf{Dataset} & \textbf{Classes} & \textbf{Features} & \textbf{$\mathbf{H}$} & \textbf{Size} \\ \midrule
     Shuttle & $7$  & $9$ & $0.342$ & $58,000$\\
     Covertype & $7 $ & $12$ & $0.619$ & $581,012$\\ 
     Diabetes & $3$  & $39$ & $0.845$ & $77,975$\\ 
     Sensorless & $11 $ & $48$ & $1.000$ & $58,509$\\
     IM & $12 $ & $85$  & $0.048$  & $3,961,221$\\  
     (1\%) IM & $12 $ & $85$  & $0.564$  & $ 62,944$\\
     \bottomrule
     \end{tabularx}
 \caption{Dataset characteristics.}
 \label{tab:dataset}
\end{table}

\subsubsection{Baselines.}
\begin{table*}[t!]
\centering
     \begin{tabular}{ c c c c c c }
     \toprule 
     \textbf{Model}& \textbf{Shuttle}&\textbf{Covertype}&\textbf{Diabetes}&\textbf{Sensorless}&\textbf{IM}\\
     \midrule
          & \multicolumn{5}{c}{\textbf{Balanced Accuracy on Test Sets}}\\

    GBT 
&$0.997\pm0.008$
&$0.938\pm0.003$
&$0.447\pm0.004$
&$0.999\pm0.000$
&$0.246\pm0.003$\\
\hline
    MC-EBM
    &$\mathbf{0.972\pm0.031}$

&$\mathbf{0.538\pm0.003}$
&$\mathbf{0.428\pm0.004}$
&$\mathbf{0.997\pm0.001}$
&$\mathbf{0.236\pm0.003}$
\\
MGCV
&$\mathbf{0.998\pm0.005}$

&$0.507\pm0.003$
&$0.332\pm0.003$
&$0.992\pm0.001$
&$0.231\pm0.002$
\\
LR
&$0.617\pm0.060$

&$0.356\pm0.004$
&$0.387\pm0.002$
&$0.832\pm0.006$
&$0.213\pm0.002$
\\\midrule

 & \multicolumn{5}{c}{\textbf{Cross-Entropy Loss on Test Sets}}\\
    GBT 
 &$0.002\pm0.000$

&$0.087\pm0.001$
&$0.821\pm0.007$
&$0.006\pm0.001$
&$0.799\pm0.009$
\\
\hline
    MC-EBM
&$\mathbf{0.001\pm0.000}$
&$\mathbf{0.608\pm0.002}$
&$\mathbf{0.840\pm0.007}$
&$\mathbf{0.017\pm0.002}$
&$\mathbf{0.829\pm0.010}$
\\
MGCV
&$\mathbf{0.001\pm0.001}$
&$0.617\pm0.002$
&$1.038\pm0.011$
&$0.036\pm0.003$
&$0.857\pm0.017$
\\
LR
&$0.208\pm0.003$
&$0.719\pm0.002$
&$0.876\pm0.006$
&$0.682\pm0.007$
&$0.892\pm0.010$
\\\bottomrule
     \end{tabular}
 \caption{Accuracy of MC-EBM compared to three baselines on five datasets.}
\label{tab:result}
\end{table*}
We compare MC-EBM to three baselines: \begin{itemize}
    \item \textbf{Multiclass logistic regression (LR)}, a simple multiclass interpretable model. This comparison tells us how much accuracy improvement is due to the non-linearity of MC-EBM. We use the \texttt{sklearn} implementation. 
    \item \textbf{Multiclass gradient boosted trees (GBT)}, an unrestricted, full-complexity model. This gives us a sense of how much accuracy we sacrifice in order to gain interpretability with GAMs. We use the \texttt{XGBoost} implementation \cite{Chen:2016:XST:2939672.2939785} and tune the hyperparameters using random search.
    \item \textbf{GAMs with splines (MGCV)}, a widely-used R package that fits GAMs with spline-based learners using a penalized likelihood procedure \cite{mgcv}. Unfortunately, as noted in the documentation\footnote{\url{https://stat.ethz.ch/R-manual/R-devel/library/mgcv/html/multinom.html}} and found by us, \texttt{mgcv}'s multiclass GAM fitting procedure does not scale beyond several thousand data points and five classes. Therefore, we trained $K$ GAMs with binary targets to predict whether a point belongs in class $k\in[K]$, then generated multiclass predictions for each point by normalizing the $K$ probabilities to sum to one.
    This comparison tells us whether our GAM learning algorithm based on boosted bagged trees is more accurate than one of the best state-of-the-art GAM implementations currently available.
\end{itemize}

\subsubsection{Experimental design.} For each dataset, we generated five train-validation-test splits of size 80\%-10\%-10\% to account for potential variability between test set splits, and report the mean and standard deviation of metrics over test set splits. 
We track two performance metrics on the test-sets: balanced accuracy and cross-entropy loss. The balanced accuracy metric addresses the imbalance of classes in classification tasks \cite{brodersen2010balanced}: $
BACC(f) = \frac{1}{K}\sum_{k=1}^K \P(f(\x) = k | y=k)
$.

\subsubsection{Results.}
\label{sec:accuracy}
The results are shown in Table \ref{tab:result}. The top half of the table reports the balanced accuracy of each model on the five datasets. The bottom half reports the cross-entropy loss on the test set. Several clear patterns emerge in both tables:\\
(1) MC-EBM consistently outperforms the LR baseline. For four out of five datasets (except for IM), the accuracy gap is larger than 5\%. This shows that the nonlinearity in MC-EBM consistently helps in fitting better models while remains interpretable.\\
(2) MC-EBM consistently outperforms MGCV across all five datasets over both metrics, showing that our implementation based on boosted trees beats a state-of-the-art GAM implementation based on splines.\\
(3) GBT, the full-complexity model still outperforms MC-EBM with restricted capacity. However, on four out of five datasets (except for Covertype), the accuracy gap between GBT and MC-EBM is smaller than 5\%. This indicates that higher order interactions, which are captured by GBT but not by GAMs, are not always helpful in predictive tasks. In some domains, an interpretable model such as GAM can achieve similar performance to a full complex model.\\
(4) Interestingly, on datasets with very imbalanced classes (IM and Shuttle), MC-EBM performs reasonably well compared to GBT, even though no explicit method countering class imbalance (e.g. loss function re-weighting) is used in MC-EBM.

In conclusion, we have presented a scalable, high-performing multiclass GAM fitting algorithm which requires little hyperparameter tuning. In the next section, we turn our attention to the interpretability of multiclass additive models.
\section{Interpretability of Multiclass Additive Models}
\label{sec:interpretability}
Multiclass GAMs are hard to interpret fundamentally because each class's prediction necessarily involves the shape functions of all $K$ classes. However, research has found that human perception cannot effectively dissect interactions between more than a few function curves~\cite{javed2010graphical}. Thus, we need to find a way to allow each shape function to be examined individually, while still conveying useful and faithful information about the model's predictions. To do so, we first revisit the binary classification setting and define what `useful and faithful information' is. Throughout this section, we will use notation defined in Section \ref{sec:notation}.

\subsection{Axioms of Interpretability: Inspiration from Binary GAMs}
\label{sec:axioms}
What information do people gain from binary shape functions and what aspect of shape functions carries that information? As demonstrated in the pneumonia example in Figure \ref{fig:age_binary}, when practitioners look at a binary GAM shape plot, they try to determine which feature values contribute positively or negatively to the outcome by looking at the monotonicity of the shape functions in different regions of the feature's domain. They also look for discontinuities in the shape functions that indicate sudden increases or decreases in the predicted probability. These sudden changes often carry rich information. For example, one might expect the influence of age on pneumonia risk to be smooth --- one's health at age 67 should not be dramatically different than at age 66 --- and the appearance of jumps may hint at the existence of hidden variables such as retirement that warrant further investigation. Because human perception naturally focuses on discontinuities in otherwise smooth curves, it is important for shape functions to be smooth when possible, so that the real discontinuities can stand out. 

In binary GAMs, the monotonicity and discontinuity of individual shape functions faithfully represent the trend and jumps of the model's predictions. We would like to be able to interpret multiclass GAMs the same way. To achieve this, we propose two interpretability axioms that every multiclass additive model should satisfy in order to be interpreted easily.\\
\textbf{A1: The axiom of monotonicity} asks that for each feature, the monotonicity of shape functions for all classes should match the monotonicity of the `average' predicted probability of that class. Mathematically:
\begin{defn}[The axiom of monotonicity]
For each class $k$, feature $i$ and feature value $v$, denote the marginal distribution of points satisfying $x_i = v$ as $\P_{x_i=v} = \P(X\vert x_i = v)$. Then, a multiclass GAM $\F$ satisfies the axiom of monotonicity if
\begin{eqnarray}
&\nabla_{x_i}f_{ik}\times\left(\E_{\P_{x_i=v}} \nabla_{x_i}\log(\hat\P_k)\right)\geq 0
\label{eq:axiom1}
\end{eqnarray}
$\forall i\in[d], k\in[K], v\in X_i$,
\end{defn}
\noindent\textbf{A2: The axiom of smoothness} asks that the shape functions do not have any artificial or unnecessary discontinuities. Mathematically:
\begin{defn}[The axiom of smoothness]
$\F$ satisfies the axiom of smoothness if
\begin{eqnarray}
\F = \argmin_{E_\F} \sum_{i\in[d]}\sum_{k\in[K]}V(f_{ik})\label{eq:axiom2}
\end{eqnarray}
where $V$ is some smoothness metric and $E_\F$ denote the equivalence class of $\F$, defined in the next section.
\end{defn}

To measure the smoothness of 1-d functions such as our shape functions, we use \textit{quadratic variation}:
\begin{defn}[Quadratic Variation]
For functions defined on a finite ordered discrete domain of size S, quadratic variation is
\begin{eqnarray}
V(f) = \sum_{s\in[S-1]} |\nabla_x f(x_{s})|^2.\nonumber
\end{eqnarray}
For functions defined on a continuous interval $[x_0,x_S]$ with finite points of discontinuity $\{x_1,...,x_{S-1}\}$, quadratic variation is:
\begin{eqnarray}
V(f) = \sum_{s=0}^{S-1} \int^{x_{s+1}}_{x_s} \left|\nabla_x f\right|^2dx + \sum_{s=1}^{S-1} |\nabla_x f(x_{s})|^2\nonumber
\end{eqnarray}
\end{defn}

Does there exist a multiclass GAM model that satisfies both axioms?  Figure \ref{fig:age_mtc} in Section \ref{sec:introduction} is an example of one. 
By transforming the binary pneumonia GAM model (Figure \ref{fig:age_binary}) to a multiclass GAM model with two classes (Figure \ref{fig:age_mtc}), the model changes from
$
\frac{1}{1+\exp\left(-\sum f_i(x_i)\right)}$ to $\frac{\exp\left(\frac{1}{2}\sum f_i(x_i)\right)}{\exp\left(\frac{1}{2}\sum f_i(x_i)\right)+\exp\left(-\frac{1}{2}\sum f_i(x_i)\right)}
$.
The blue curve representing risk of dying is exactly the same as the binary age shape and is therefore faithful to the model prediction.
The orange curve representing the 'risk' of surviving is exactly the mirror image of the risk of dying. Since in the binary case the probability of dying is always one minus the chance of surviving, the orange curve is faithful to its own class as well. Does this generalize to settings with more than two classes? The answer is YES.

\subsection{Leveraging Key Properties of Multiclass GAMs to Regain Interpretability}
\label{sec:properties}
We have proposed two axioms satisfied by binary GAMs that multiclass GAMs should also satisfy in order to not be visually misleading, and provided an example of a (two-class) multiclass GAM model that satisfies these axioms. We now highlight two key properties shared by all multiclass GAM models that we will leverage in Section \ref{sec:api} to post-process {\em any} multiclass GAM model to satisfy these axioms. These properties stem from the softmax formulation (Equation \eqref{eq:mtc_gam}) used by these models. 

\noindent\textbf{P1: Equivalence class of multiclass GAMs.}
Different GAMs can produce equivalent model predictions.
In particular, we have the following equivalence relationship:
\begin{prop}\label{equivalence} Let $\F$ and $\F'$ be two GAMs defined as
\begin{eqnarray}
\F &=& \{f_{ij} \, | \, i\in[d], k\in[K]\},\nonumber \\
\F'&=& \{f_{ij}+g_{i} \, | \, i\in[d], k\in[K]\},\nonumber
\end{eqnarray}
for some arbitrary functions $g_i$'s. Then, $\F$ and $\F'$ are equivalent in terms of model prediction, and we define the \textit{equivalence class} of $\F$ as $E_\F = \{\F'\vert \F'\equiv \F\}$.
\end{prop} 
\begin{proof}
Notice that unlike the binary GAMs' logistic probabilities, softmax probabilities are invariant with respect to a constant shift of the logits due to the softmax being overparametrized.
Therefore we can add a constant $g_i$ to all $K$ logits without changing the predicted probability, i.e.
\begin{align}
\hat\P(y=k) &= \frac{\exp\left(\sum_{i=1}^df_{ik}(x_i)\right)}{\sum_{j=1}^{K}\exp\left(\sum_{i=1}^df_{ij}(x_i)\right)}\nonumber\\
&= \frac{\exp\left(\sum_{i=1}^df_{ik}(x_i)+\sum_{i=1}^dg_{i}(x_i)\right)}{\sum_{j=1}^{K}\exp\left(\sum_{i=1}^df_{ij}(x_i)+\sum_{i=1}^dg_{i}(x_i)\right)}\nonumber
\end{align}
\end{proof}
We will use this invariance property in our additive post-processing (API) method presented in Section \ref{sec:api} 
to find a more interpretable $\F'$ equivalent to $\F$.

\noindent\textbf{P2: Ranking consistency between shape functions and class probabilities.}
Another characteristic of the softmax is the ranking consistency between the change in shape function values and the change in predicted class probability:
\begin{prop}
\label{cor:rankingconsistency}
Let $\x=(x_1,...,x_i,...,x_d)$ and $\x' = (x_1,...,x_i',...,x_d)$ be two data points sharing the exact same feature values except for one particular feature $i$. Let $\{\delta_j\}_1^K$ be the differences between their corresponding logits due to the difference in feature $i$. Then, the ranking of $\{\delta_j\}_1^K$ across $j$ is consistent with the ranking of the ratios of predicted probabilities $\left\{\frac{\hat\P_j(\x')}{\hat\P_j(\x)}\right\}_1^K$ across $j$.
\end{prop}
\begin{proof}
Simple calculation shows that $\delta_j = f_{ij}(x_i') - f_{ij}(x_i)$, for all $j$. Now, suppose that $\delta_j\geq\delta_k$ for some particular $j,k\in[K]$, then we have
\begin{eqnarray}
\frac{\hat\P_j(\x')}{\hat\P_k(\x')} = \frac{\hat\P_j(\x)}{\hat\P_k(\x)}\cdot\frac{\exp(\delta_j)}{\exp(\delta_k)}\geq \frac{\hat\P_j(\x)}{\hat\P_k(\x)}
\end{eqnarray}
which implies that 
\begin{eqnarray}\label{eqn:inequality}
\frac{\hat\P_j(\x')}{\hat\P_j(\x)}\geq
\frac{\hat\P_k(\x')}{\hat\P_k(\x)}.
\end{eqnarray}
This property holds for all $(j,k)$ pairs.
\end{proof}
This ranking consistency property will come in useful in the optimization of our API method (cf. Section \ref{sec:api}).

\begin{figure*}[ht!]
	\centering
	\begin{minipage}[t]{\textwidth}
		\centering
		\includegraphics[width=\columnwidth]{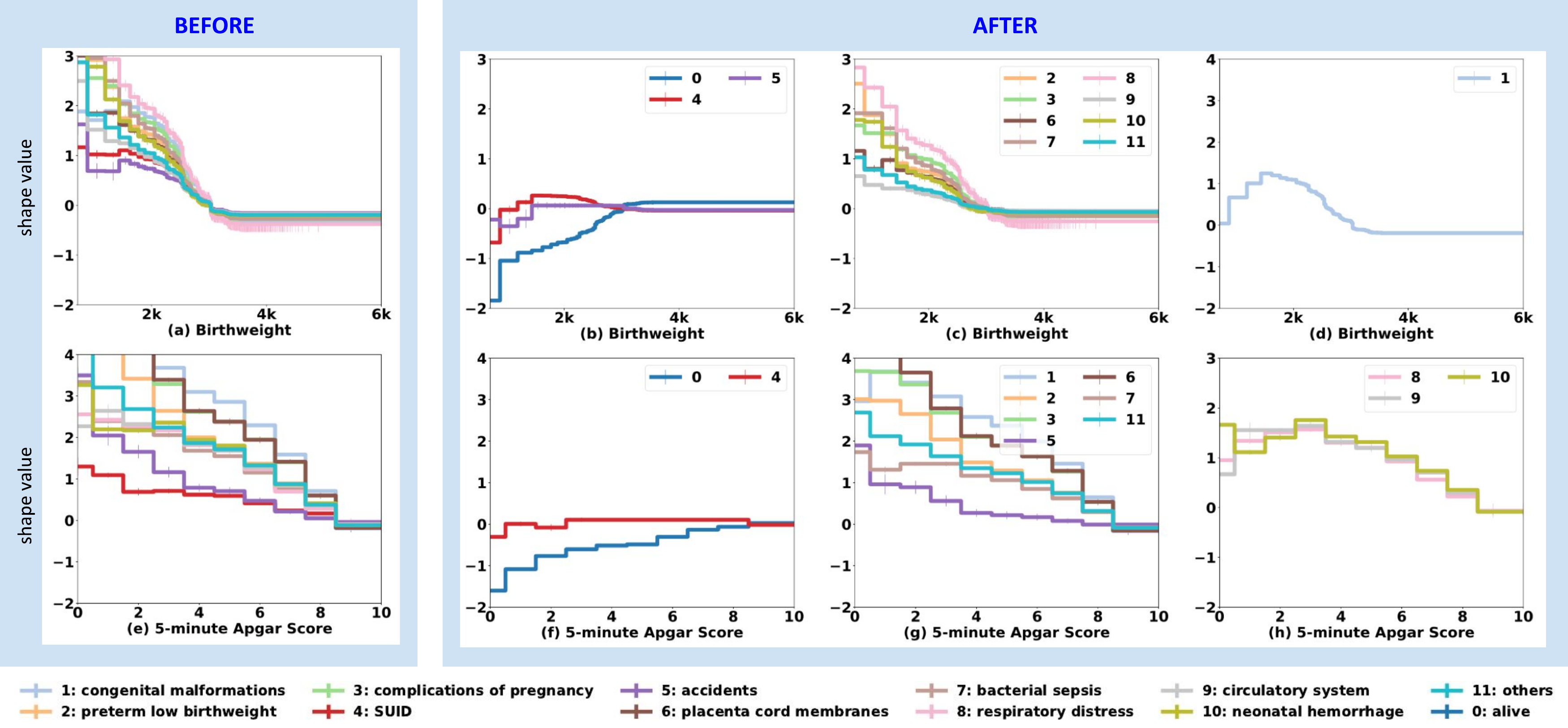}
	\end{minipage}
	\caption{Shape functions for the IM data, before and after applying our API post-processing method.}
	\label{fig:sids}
\end{figure*}
\subsection{Additive Post-Processing for Interpretability}
\label{sec:api}
We now describe our post-processing method, API, that leverages the softmax's properties (cf. Section \ref{sec:properties}) to modify any multiclass additive model to regain interpretability (cf. Section \ref{sec:axioms}), while keeping its predictions unchanged. Given a pretrained GAM model $\F$, API finds another equivalent additive model $\F'$ that satisfies the axiom of monotonicity while fulfilling the minimization condition of the axiom of smoothness. We formulate this as a constrained optimization problem in functional space to find the set 
$\{g_1,...,g_d\}$ defining $\F'$ while minimizing objective \eqref{eq:axiom2} and satisfying condition \eqref{eq:axiom1}:
\begin{eqnarray}
\min_{g_1,...,g_d}&& \sum_{i\in[d]}\sum_{k\in[K]}V(f_{ik}+g_{i})\label{eq:objective}\\
\mbox{s.t.}&& (\nabla_{x_i}f_{ik}+\nabla_{x_i}g_i)\cdot\left(\E_{\P_{x_i=v}} \nabla_{x_i}\log(\hat\P_k)\right)\geq 0 \nonumber\\
&&\forall i\in[d], k\in[K], v\in X_i\label{eq:constraint}
\end{eqnarray}
Before we discuss how to solve this optimization problem, we first show that there is a solution: 
\begin{thm}
Condition \eqref{eq:constraint} is feasible.
\end{thm}
\begin{proof}
Let $i$ be a feature and $\x$ be a data point with $x_i = v$.
Here, we only present the proof for the case where the domain of feature $i$ is continuous and the shape functions $\{f_{ij}\}$ are differentiable at $x_i=v$. The proofs for the other two cases are similar.

Applying the definition of $\nabla$, we have
\begin{eqnarray}
\nabla_{x_i}\log(\hat\P_k) &=& \lim_{\Delta x \rightarrow 0}\frac{1}{\Delta x}\left[\frac{\hat\P_k(v+\Delta x)}{\hat\P_k}-1\right]\nonumber\\
\nabla_{x_i}f_{ik} &=& \frac{1}{\Delta x}\left[f_{ik}(v+\Delta x) - f_{ik}(v)\right]\nonumber
\end{eqnarray}
The ranking consistency property (Corollary \ref{cor:rankingconsistency}) therefore guarantees that the ranking among $\nabla_{x_i}f_{ik}$ is the same as the ranking among $\nabla_{x_i}\log(\hat\P_k)$. This is true for every individual data point with $x_i = v$. Then, due to the invariance of the inequality under expectation, we have that the ranking among $\nabla_{x_i}f_{ik}$ is the same as the ranking among $\E_{\P_{x_i=v}}\nabla_{x_i}\log(\hat\P_k)$. Therefore, there must exist a constant $\nabla g_i(v)$ such that the sign of $\nabla_{x_i}f_{ik}(v)+\nabla g_{i}(v)$ equals the sign of $\E_{\P_{x_i=v}}\nabla_{x_i}\log(\hat\P_k)(v)$ for all $k\in[K]$. This holds for all features $i\in[d]$ and values $v\in X_i$. Therefore, Condition \eqref{eq:constraint} is feasible.
\end{proof}
\begin{algorithm}[ht]
	\caption{Additive Post-Processing for Interpretability (API)}\label{alg:api}
	\begin{flushleft}
        \textbf{INPUT:} A pretrained GAM $\F = \{f_{ij}\}$.\\
        \textbf{OUTPUT:} Interpretable GAM $\F'$.
\end{flushleft}
	\begin{algorithmic}[1]
	\For{$i = 1$ to $d$}
	\For{$k = 1$ to $K$}
	\State Define function $\bar p_{ik}(v) = \E_{\P_{x_i=v}}\nabla_{x_i}\log(\hat\P_k)$.
	\EndFor
	\State Define function $\bar f_i = \frac{1}{K} \sum_{k=1}^K f_{ik}$.
	\State Define function $J^+_i = \argmin_{k\in [K]\mbox{, } \bar p_{ik}\geq 0} \bar p_{ik}$.
	\State Define function $J^-_i = \argmax_{k\in [K]\mbox{, } \bar p_{ik}< 0} \bar p_{ik}$.
    \State $\nabla g_i \gets \max\left(-f_{iJ_i^+},\min\left(-\bar f_i,-f_{iJ_i^-}\right)\right)$.
    \State Recover $g_i$ via integration or summation depend on the domain type of $f_{ij}$.
    \EndFor
    \State Return $\F' = \{f_{ij}+g_i\}$.
	\end{algorithmic}
\end{algorithm}
Now to solve optimization problem \eqref{eq:objective}, observe that both the objective function and the constraints are separable with respect to the feature set $i\in[d]$ and the feature values $v\in X_k$, and the optimization problem can be reparametrized to be a problem over $\nabla_{x_i} g_i(v)$. Therefore, problem \eqref{eq:objective} can be solved by individually solving
\begin{eqnarray}
\min_{\nabla_{x_i}g_i(v)} &&\sum_{k=1}^{K} \left|\nabla_{x_i}f_{ik}(v)+\nabla_{x_i}g_i(v)\right|^2\label{eq:individual_objective}\nonumber\\
\mbox{s.t.} &&\nabla_{x_i}(f_{ik}+g_i)(v)\left(\E_{\P_{x_i=v}} \nabla_{x_i}\log(\hat\P_k)\right)\geq 0 \nonumber\\
&&\forall k\in[K],\nonumber
\end{eqnarray}
for all $i\in[d]$ and $v\in X_k$. It therefore becomes a set of 1-d quadratic programs with linear constraints, which can be solved in closed form. The closed form solution gives rise to the API post-processing method presented in Algorithm \ref{alg:api}.

In the next section, we present a case study in which we apply API to the shape functions of a multiclass GAM model trained on a 12-class infant mortality dataset, and show that, with the help of API, the shape functions reveal interesting patterns in the learned model that would otherwise be difficult to see.

\subsection{Interpretability in Action on Real Data: Infant Mortality Dataset (IM)}
\label{sec:sids}
The IM dataset \cite{cdcIM} contains data on all live births in the United States in 2011.  It classifies newborn infants into 12 classes: alive, top 10 distinct causes of death (see Figure~\ref{fig:sids} legend), and death due to other causes. The usual way of visualizing multiclass additive models, used in packages such as \texttt{mgcv} \cite{mgcv}, plots the logit relative to a \textit{base class} that is the majority or `normal' class: in IM the class `alive' is the natural base class. Note that this post-processing forces the logit for class `alive' to zero for all values of each feature so that the risk of other classes is relative to the `alive' class. 

The first column in Figure~\ref{fig:sids} shows this view of the shape functions for features `birthweight' and `apgar' denoting the weight of the infant at birth and the 5-minute Apgar score (on a scale of 0-10) capturing the infant's general health after the first five minutes of life . Interpreting the model from these two plots (Figure \ref{fig:sids}(a),(e)), one may conclude that the risk for almost all causes of death is high for infants with low birthweight or low Apgar score, since all 11 curves in both plots are monotonically decreasing as birthweight rises from 0 to 3000g and as the Apgar score rises from 0 to 9. However, as pointed out in the beginning of Section~\ref{sec:interpretability}, shape functions without applying API will generally not represent the actual predicted probabilities of the corresponding classes. These shapes only represent the \textit{relative} probability between each cause of death with respect to being alive. However, as we will soon see, the relative probability can disagree dramatically with the actual predicted probability for each cause of death.
In fact, a medical expert who was invited to examine these two plots, found them misleading and questioned ``why risk did not appear to differ more by cause of death''.

The three columns on the right show the shape functions for the same two features, `birthweight' and `apgar', after applying the API method. For the sake of demonstration, for each feature we split the 12 shapes into three figures. Keep in mind that after API post-processing, the trend of the shapes agrees with the trend of the corresponding class probabilities. One can see that the chance of living (class 0) is indeed monotonically decreasing as birthweight and the Apgar score get lower  (Figure~\ref{fig:sids}(b),(f)). However, not all causes of death are affected in the same way by the two features.

Low birthweight infants are more likely to die from complications related to preterm birth and/or low birthweight status, complications of pregnancy, problems related to placenta, cord, and membranes, from respiratory distress, bacterial sepsis, neonatal hemorrhage and (to a lesser degree) circulatory system problems (2-3,6-10 in Figure~\ref{fig:sids}(c)), while the risk of low birthweight infants dying from SUID (sudden unexpected infant death) is only slightly elevated, and the risk of dying from accidents is actually lower for the smallest babies (4,5 in Figure~\ref{fig:sids}(b)). For congenital malformations, the risk peaks at birthweight 1.5kg (1 in Figure~\ref{fig:sids}(d)), but drops as birthweight gets even smaller. These observations were confirmed by medical experts and agree with known domain knowledge.

For the Apgar score, the causes of death exhibit three different patterns. As the score gets lower, we observe increased risk of death from congenital malformations, complications due to preterm birth and/or low birthweight, complications of pregnancy, problems related to placenta, cord, and membranes and bacterial sepsis (1-3,5-7 in Figure~\ref{fig:sids}(g)). SUID is least affected by the Apgar score (4 in Figure~\ref{fig:sids}(f)). The 3rd category (Figure~\ref{fig:sids}(h)) is especially interesting. The risk of death from respiratory distress, circulatory system problems and neonatal hemorrhage appear to all peak around Apgar score of 3-4.

This short case study demonstrates that multiclass GAM shape functions are more readily interpretable after API (three columns on the right in Figure~\ref{fig:sids}) compared to the traditional presentation (column on the left in Figure~\ref{fig:sids}). In particular, the shape plots after API successfully show the diversity between different causes of death that is not immediately apparent in the plots before API.

\section{Discussion and Conclusions}

We have presented a comprehensive framework for constructing interpretable multiclass generalized additive models. The framework consists of a multiclass GAM learning algorithm, MC-EBM, and a model-agnostic post-processing procedure, API, that transforms any multiclass additive model into a more interpretable, canonical form. The API post-processing method provably satisfies two interpretability axioms that, when satisfied, allow the learned shape functions to be looked at individually and prevent them from being visually misleading. The API method is general, and can also be applied to simple additive models such as multiclass logistic regression to create a more interpretable, canonical form.

The MC-EBM algorithm and API post-processing method are efficient and easily scale to large datasets with hundreds of thousands of points and hundreds or thousands of features. We are currently generalizing both the MC-EBM algorithm and API post-processing method to work with GAMs that include higher-order interactions such as pairwise interactions. 

Even though this work focuses primarily on training interpretable models from ground-up, the challenge of interpreting multi-class predictions addressed in this paper and the corresponding solution might also benefit explanation methods for black-box models. In particular, explanation methods using model distillation, such as LIME \cite{lime}, often use simple linear models as the student model to produce a local interpretable approximation to the otherwise complex black-box model. However, when the problem is multiclass and when the user is interested in interpreting the prediction of several classes simultaneously, the same problem would arise, and the same solution, API, applies.
\subsection*{Acknowledgement}
We thank Dr. Ed Mitchell from the University of Auckland for valuable feedback on our algorithm's results on the infant mortality dataset. Xuezhou Zhang worked on this project during an internship at Microsoft Research.
\bibliographystyle{ACM-Reference-Format}
\bibliography{nllb}
\end{document}